%% file: main.tex
\documentclass[10pt, onecolumn]{ieeeconf}
\usepackage[utf8]{inputenc}

\usepackage[bookmarks=true]{hyperref}
\usepackage{xcolor}
\usepackage{graphicx}
\usepackage{dsfont} 
\usepackage{amsmath}
\usepackage{amssymb}
\newtheorem{remark}{Remark}
\usepackage{url}

\usepackage{color}
\usepackage{subcaption}
\usepackage[utf8]{inputenc}
\usepackage[T1]{fontenc}
\usepackage{balance}
\IEEEoverridecommandlockouts

\usepackage{cleveref}

\newtheorem{assumption}{Assumption}

\newtheorem{lemma}{Lemma}
\newtheorem{theorem}{Theorem}

\usepackage{comment}
\usepackage{mathrsfs}
\usepackage[textfont=md,font=footnotesize]{caption}

\newcommand{\doi}[1]{\href{http://dx.doi.org/#1}{\normalsize{\textsc{doi:}}~\nolinkurl{#1}}}
\newcommand{\arxiv}[1]{\href{http://arxiv.org/abs/#1}{\normalsize{\textsc{arxiv:}}~\nolinkurl{#1}}}

\renewcommand{\epsilon}{\varepsilon}

\newcommand{\R}{\mathbb{R}}

\newcommand{\N}{\mathbb{N}}

\let\originalleft\left
\let\originalright\right
\renewcommand{\left}{\mathopen{}\mathclose\bgroup\originalleft}
\renewcommand{\right}{\aftergroup\egroup\originalright}

\def\clap#1{\hbox to 0pt{\hss#1\hss}}

\newcommand{\norm}[1]{\left\lVert #1\right\rVert}

\newcommand{\set}[1]{\left\{ #1\right\}}

\DeclareMathAlphabet{\mathpzc}{OT1}{pzc}{m}{it}

\newcommand{\dfknote}[1]%
    {\textcolor{orange}{\textbf{DFK: #1}}}
\newcommand{\twnote}[1]%
    {\textcolor{cyan}{\textbf{TW: #1}}}
\newcommand{\emnote}[1]%
    {\textcolor{blue}{\textbf{EM: #1}}}
\newcommand{\vpnote}[1]%
    {\textcolor{green}{\textbf{VP: #1}}}

\title{Technichal Report: Adaptive Control for Linearizable Systems \\ Using On-Policy Reinforcement Learning}
\date{August 2019}

\author{
Tyler Westenbroek, Eric Mazumdar, David Fridovich-Keil, Valmik Prabhu,\\  Claire J. Tomlin and S. Shankar Sastry
\thanks{The authors are with the department of Electrical Engineering and Computers Sciences and the University of California, Berkeley.}}
\begin{document}
\maketitle

\begin{abstract}
\input{abstact.tex}
\end{abstract}
\section{Introduction}
\label{sec:intro}
\input{intro.tex}

\section{Feedback Linearization} \label{sec:FBL}
\input{fbl}
 
\section{Adaptive Control}
\label{sec:adaptive}
\input{adaptive.tex}

\section{Numerical Example}
\label{sec:examples}

\input{examples.tex}

\section{Conclusion}
\label{sec:conclusion}
\input{conclusion}

\appendix
\input{proofs}

\newpage
\bibliographystyle{IEEEtran}
\bibliography{references.bib}
\end{document}

%% file: abstact.tex
This paper proposes a framework for adaptively learning a feedback linearization-based tracking controller for an unknown system using discrete-time model-free policy-gradient parameter update rules. The primary advantage of the scheme over standard model-reference adaptive control techniques is that it does not require the learned inverse model to be invertible at all instances of time. This enables the use of general function approximators to approximate the linearizing controller for the system without having to worry about singularities. However, the discrete-time and stochastic nature of these algorithms precludes the direct application of standard machinery from the adaptive control literature to provide deterministic stability proofs for the system. Nevertheless, we leverage these techniques alongside tools from the stochastic approximation literature to demonstrate that with high probability the tracking and parameter errors concentrate near zero when a certain persistence of excitation condition is satisfied. A simulated example of a double pendulum demonstrates the utility of the proposed theory. \footnote{This draft corrects an important which appeared in an earlier draft. In particular, the right hand side of \eqref{eq:total_var} was originally $C_2 M\sqrt{\frac{ \Delta t \ln\left( \frac{\lambda}{2}\right)}{\zeta \sigma^2}}$ but has now been corrected to $C_2 M\sqrt{\frac{ \Delta t \ln\left( \frac{2}{\lambda}\right)}{\zeta \sigma^2}}$ to display the correct dependence on the confidence paramter $\lambda$. }

%% file: intro.tex
Many real-world control systems display nonlinear behaviors which are difficult to model, necessitating the use of control architectures which can adapt to the unknown dynamics online while maintaining certificates of stability. 
There are many successful model-based strategies for adaptively constructing controllers for uncertain systems \cite{sastry1989adaptive, sastry2013nonlinear, slotine1987adaptive}, but these methods often require the presence of a simple, reasonably accurate parametric model of the system dynamics.
Recently, however, there has been a resurgence of interest in the use of model-free reinforcement learning techniques to construct feedback controllers without the need for a reliable dynamics model \cite{schulman2015trust, schulman2017proximal, lillicrap2015continuous}. As these methods begin to be deployed in real world settings, a new theory is needed to understand the behavior of these algorithms as they are integrated into safety-critical control loops. 

However, the majority of the theory for adaptive control is stated in continuous-time \cite{sastry2013nonlinear}, while reinforcement learning algorithms are typically implemented and studied in discrete-time settings \cite{sutton2018reinforcement, borkar2009stochastic}. There have been several attempts to define and study policy-gradient algorithms in continuous-time \cite{munos2006policy, doya2000reinforcement}, yet many real-world systems have actuators which can only be updated at a fixed maximum sampling frequency. 
Thus, we find it more natural and practically applicable to unify these methods in the sampled-data setting. 

Specifically, this paper addresses the model mismatch issue by combining continuous-time adaptive control techniques with discrete-time model-free reinforcement learning algorithms to learn a feedback linearization-based tracking controller for an unknown system, online. Unfortunately, it is well-known that sampling can destroy the affine relationship between system inputs and outputs which is usually assumed and then exploited in the stability proofs from the adaptive control literature \cite{grizzle1988feedback}. 
To overcome this challenge, we first ignore the effects of sampling and design an idealized continuous-time behavior for the system's tracking and parameter error dynamics which employs a least-squares gradient following update rule. In the sampled-data setting, we then use an Euler approximation of the continuous-time reward signal and implement a policy-gradient parameter update rule to produce a noisy approximation to the ideal continuous-time behavior. Our framework is closely related to that of \cite{westenbroek2019feedback}; however, in this paper we address the problem of online adaptation of the learned parameters whereas \cite{westenbroek2019feedback} considers a fully offline setting.

Beyond naturally bridging continuous-time and sampled-data settings, the primary advantage of our approach is that it does not suffer from the ``loss of controllability'' phenomena which is a core challenge in the model-reference adaptive control literature \cite{sastry1989adaptive, kosmatopoulos1999switching}. This issue arises when the parameterized estimate for the system's decoupling matrix becomes singular, in which case either the learned linearizing control law or associated parameter update scheme may break down. To circumvent this issue, projection-based parameter update rules are used to keep the parameters in a region in which the estimate for the decoupling matrix is known to be invertible. In practice, the construction of these regions requires that a simple parameterization of the system's nonlinearities is available \cite{craig1987adaptive}. In contrast, the model-free approach we introduce does not suffer from singularities and can naturally incorporate `universal' function approximators such as radial bases functions or bases of polynomials. 

However, due to the non-deterministic nature of our sampled-data control law and parameter update scheme, the deterministic guarantees usually found in the adaptive control literature do not apply here. Indeed, policy-gradient parameter updates are known to suffer from high variances \cite{zhao2011analysis}. Nevertheless, we demonstrate that when a standard persistence of excitation condition is satisfied the tracking and parameter errors of the system concentrate around the origin with high probability even when the most basic policy-gradient update rule is used. Our analysis technique is derived from the adaptive control literature and the theory of stochastic approximations \cite{borkar2009stochastic, vershynin2018high}.  Proofs of claims made can be found in the Appendix of the document. Finally, a simulation of a double pendulum demonstrates the utility of the approach.

\subsection{Related Work}
A number of approaches have been proposed to avoid the ``loss of controllability'' problem discussed above. One approach is to perturb the estimated linearizing control law to avoid singularities \cite{kosmatopoulos1999switching, kosmatopoulos2002robust, bechlioulis2008robust}. However, this method never learns the exact linearizing controller during operation and hence sacrifices some tracking performance. Other approaches avoid the need to invert the input-output dynamics by driving the system states to a sliding surface \cite{slotine1987adaptive}. Unfortunately, these methods require high-gain feedback which may lead to undesirable effects such as actuator saturation. Several model-free approaches similar to the one we consider here have been proposed in the literature \cite{hwang2003reinforcement,zomaya1994reinforcement}, but these focus on actor-critic methods and, to the best of our knowledge, do not provide any proofs of convergence.  Recently, non-parametric function approximators have been been used to learn a linearizing controller \cite{umlauft2017feedback, umlauft2019feedback}, but these methods still require structural assumptions to avoid singularities.

While our parameter-update scheme is most closely related to the policy gradient literature, e.g., \cite{sutton2018reinforcement}, we believe that recent work in meta-learning \cite{finn2017model, santoro2016meta} is also similar to our own work, at least in spirit.
Meta-learning aims to learn priors on the solution to a given machine learning problem, and thereby speed up online fine tuning when presented with a slightly different instance of the problem \cite{vilalta2002perspective}.
Meta-learning is used in practice to apply reinforcement learning algorithms in hardware settings \cite{nagabandi2018learning, andrychowicz2020learning}. 

\subsection{Preliminaries}
Next, we fix mathematical notation and review some definitions used extensively in the paper. Given a random variable $X$, if they exist the expectation of $X$ is denoted $\mathbb{E}[X]$ and its variances is denoted by $Var(X)$. Our analysis heavily relies on the notion of a \emph{sub-Gaussian} distribution. We say that a random variable $X \in \R^n$ is sub-Gaussian if there exists a constant $C>0$ such that for each $t \geq 0$ we have $\mathcal{P}\set{|x|_2 \geq t} \leq 2 \exp(-\frac{t^2}{C^2})$. Informally, a distribution is sub-Gaussian if it's tail is dominated by the tail of some Gaussian distribution. We endow the space of sub-Gaussian distributions with the norm $\norm{\cdot}_{\psi_2}$ defined by $\norm{X}_{\psi_2} = \inf \set{t>0 \colon \mathbb{E}[\exp(\frac{\norm{X}_2^2}{t^2})] \leq 2}$. As an example, if $X = \mathcal{N}(0, \sigma^2I)$ is a zero-mean Gaussian distribution with variance $\sigma^2 I$ (with $I$ the $n$-dimensional identity) then $\norm{X}_{\psi_2}$ is sub-Gaussian with norm $\norm{X}_{\psi_2} \leq C\sigma$, where the constant $C>0$  does not depend on $\sigma^2$.

%% file: fbl.tex
Throughout the paper we will focus on constructing output tracking controllers for systems of the form
\begin{align}\label{eq:nonlinear_sys}
    \dot{x} &= f(x) + g(x)u\\
    y &= h(x) \nonumber
\end{align}
where $x \in \mathbb{R}^n$ is the state, $u \in \R^q$ is the input and $y \in \mathbb{R}^q$ is the output.  The mappings $f \colon \R^n \to \R^n$, $g \colon \R^n \to \R^{n \times q}$ and $h \colon \R^{n}\to \R^q$ are each assumed to be smooth, and we assume without loss of generality that the origin is an equilibrium point of the undriven system, i.e., $f(x) = 0$. Throughout the paper, we will also assume that state $x$ and the output $y$ can both be measured.

\subsection{Single-input single-output systems}
We begin by introducing feedback linearization for single-input, single-output (SISO) systems (i.e.,  $q =1 $). We begin by examining the first time derivative of the output:
\begin{align}\label{eq:nominal}
\dot{y} &= L_fh(x) + L_gh(x) 
\end{align}
Here the terms $L_{f}h(x) = \frac{d}{dx}h(x) \cdot f(x)$ and $L_{g}h(x) =\frac{d}{dx}h(x)\cdot g(x) $ are known as \emph{Lie derivatives} \cite{sastry2013nonlinear}. In the case that $L_{g}h(x) \neq0$ for each $x \in \R^n$, we can apply 
\begin{equation}\label{eq:fb1}
    u(x,v) = \frac{1}{L_{g}h(x)} (-L_{f}h(x) + v)~,
\end{equation}
which exactly `cancels out' the nonlinearities of the system and enforces the linear relationship $\dot{y} = v$ with $v$ some arbitrary, auxiliary input. However if the input does not affect the first time derivative of the output---that is, if $L_gh \equiv 0$---then the control law \eqref{eq:fb1} will be undefined. In general, we can differentiate $y$ multiple times, until the input shows up in one of the higher derivatives of the output. Assuming that the input does not appear the first $\gamma-1$ times we differentiate the output, the $\gamma$-th time derivative of $y$ will be of the form
\begin{equation}
    y^{(\gamma)} = L_f^{\gamma}h(x) + L_gL_f^{\gamma-1}h(x)u
\end{equation}
Here, $L_f^{\gamma}h(x)$ and $L_gL_f^{\gamma-1}h(x)$ are higher order Lie derivatives, and we direct the reader to \cite[Chapter 9]{sastry2013nonlinear} for further details. If $L_gL_f^{\gamma-1}h(x) \neq0$ for each $x \in \R^n$ then setting
\begin{equation}
    u(x,v) = \frac{1}{L_gL_f^{\gamma -1}h(x)}\big(-L_f^{\gamma}h(x) + v\big)
\end{equation}
enforces the trivial linear relationship $y^{\gamma} = v$. We refer to $\gamma$ as the \emph{relative degree} of the nonlinear system, which is simply the order of its input-output relationship. 

\subsection{Multiple-input multiple-output systems}\label{subsec:mimo}
Next, we consider square multiple-input, multiple-output (MIMO) systems where $q > 1$.  As in the SISO case,  we differentiate each of the output channels until at least one input appears.  Let $\gamma_j$ be the number of times we need to differentiate $y_j$ (the $j$-th entry of $y$) for at least one input to appear. Combining the resulting expressions for each of the outputs yields an input-output relationship of the form
\begin{equation}
\label{eq:first_A_b}
  y^{(\gamma)} = b(x) +  A(x) u
\end{equation}
where we have adopted the shorthand $y^{(\gamma)} = [y_1^{(\gamma_1)}, \dots, y_q^{(\gamma_q)}]^T$. Here, the matrix $A(x) \in \R^{q \times q}$ is known as the \emph{decoupling matrix} and the vector $b(x) \in \mathbb{R}^q$ is known as the \emph{drift term}. If $A(x)$ is non-singular on for each $x \in \R^n$ then we observe that the control law
\begin{equation}\label{eq:mimo_controller}
    u(x,v) = A^{-1}(x)(-b(x) + v)
\end{equation}
where $v \in \R^q$ yields the decoupled linear system
\begin{equation}\label{eq:decoupled_sys}
    [y_1^{(\gamma_1)}, y_2^{(\gamma_2)}, \dots, y_q^{(\gamma_q)}]^T = [v_1, v_2, \dots, v_q]^T,
\end{equation}
where $v_k$ is the $k$-th entry of $v$ and $y_j^{\gamma_j}$ is the $\gamma_j$-th time derivative of the $j$-th output. We refer to $\gamma =(\gamma_1, \gamma_2, \dots, \gamma_q)$ as the \emph{vector relative degree} of the system, with $|\gamma| = \sum_i \gamma_i$ the total relative degree of all dimensions. The decoupled dynamics \eqref{eq:decoupled_sys} can be compactly represented with the LTI system
\begin{equation}
    \label{eq:reference}
    \dot{\xi}_r = A\xi_r + Bv_r
\end{equation}
which we will hereafter refer to as the \emph{reference model}. Here, $A \in \R^{|\gamma| \times |\gamma|}$ and $B \in \R^{|\gamma| \times q}$ is constructed so that $B^T B = I_{q\times q}$, where $I_{q\times q}$ is the $q$-dimensional identity matrix. Note that \eqref{eq:reference} collects $\xi_r = (y_1, \dot{y}_1, \dots, \dots , y_1^{\gamma_1-1}, \dots, y_q, \dots, y_q^{\gamma_q-1})$. It can be shown \cite[Chapter 9]{sastry2013nonlinear} that there exists a change of coordinates $x \to (\xi,\eta)$ such that in the new coordinates and after application of the linearizing control law the dynamics of the system are of the form
\begin{align}\label{eq:zero_dynamics}
    \dot{\xi} &= A \xi + Bv \\ \nonumber
    \dot{\eta}& = q(\xi,\eta) + p(\xi,\eta)v.
\end{align}
That is, the $\xi \in \mathbb{R}^{|\gamma|}$ coordinates represent the portion of the system that has been linearized while the $\eta \in \mathbb{R}^{n - |\gamma|}$ coordinates represent the remaining coordinates of the nonlinear system. The undriven dynamics
\begin{equation}
    \dot{\eta} = q(\xi,\eta)
\end{equation}
are referred to as the \emph{zero} dynamics. Conditions which ensure that the $\eta$ coordinates remain bounded during operation will be discussed below.

\subsection{Inversion \& exact tracking for min-phase MIMO systems}\label{subec:tracking}
Let us assume that we are given a desired reference signal $y_d(\cdot) = \big(y_{1,d}(\cdot), \dots, y_{q,d}(\cdot)\big)$. Our goal is to construct a tracking controller for the nonlinear system using the linearizing controller \eqref{eq:mimo_controller}, along with a linear controller designed for the reference model \eqref{eq:reference} which makes use of both feedback terms. We will assume that the first $\gamma_j$ derivatives of $y_{j,d}(\cdot)$ are well defined, and assume that the signal $\big(y_{j,d}(\cdot), y_{j,d}^{(1)}(\cdot), \dots, y_{q,d}^{(\gamma_q)}(\cdot)\big)$ can be bounded uniformly.

For compactness of notation, we will collect
\begin{equation*}
y_d^{(\gamma)}(\cdot) = \big(y_{1,d}^{(\gamma_1)}(\cdot), y_{2,d}^{(\gamma_2)}(\cdot), \dots, y_{q,d}^{(\gamma_q)}(\cdot)\big)
\end{equation*}
\begin{equation*}
 \xi_d(\cdot) = \big(y_{1,d}(\cdot), \dots, y_{1,d}^{(\gamma_1 - 1)}(\cdot), \dots, y_{q,d}(\cdot), \dots , y_{q,d}^{(\gamma_q-1)}(\cdot)).
 \end{equation*}
Here, $\xi(\cdot)$ is used to capture the desired trajectory of the linear reference model, and $y_d^{(\gamma)}(\cdot)$ will be used in a feedforward term in the tracking controller. To construct the feedback term, we define the error
\begin{equation}\label{eq:error}
    e(\cdot)=  \xi(\cdot) -\xi_d(\cdot)
\end{equation}
where $\xi(\cdot)$ is the actual trajectory of the linearized coordinates as in \eqref{eq:zero_dynamics}. Altogether, the tracking controller for the system is then given by
\begin{equation}
    u = A^{-1}(x)\big(-b(x) + y_d^{(\gamma)} + Ke\big)
\end{equation}
where $K \in \R^{q \times |\gamma|}$ is a linear feedback matrix designed so that $(A +BK)$ us Hurwitz. Under the application of this control law the closed loop error dynamics become
\begin{equation}\label{eq:error_dyn}
    \dot{e} = (A+BK)e
\end{equation}
and it becomes apparent that $e \to 0$ exponentially quickly. However, while the tracking error decays exponentially, the $\eta$ coordinates may be come unbounded during operation, in which case the linearizing control law will break down. One sufficient condition for $\eta$ to remain bounded is for the zero dynamics to be globally exponentially stable and for $\xi_d(\cdot)$ and $y_d(\cdot)$ to remain bounded \cite[Chapter 9]{sastry1989adaptive}. When the zero dynamics satiecfy this condition we say nonlinear system is \emph{exponentially minimum phase}. 

%% file: adaptive.tex
From here on, we will aim to learn a feedback linearization-based tracking controller for the unknown plant
\begin{align}\label{eq:plant_dynamics}
    \dot{x}_p &= f_p(x_p) + g_p(x_p)u_p\\ \nonumber
    y_p &= h_p(x_p)
\end{align}
in an adaptive fashion. We assume that we have access to a an approximate dynamics model for the plant
\begin{align}\label{eq:model_dynamics}
    \dot{x}_m &= f_m(x_m) + g_m(x_m)u_m\\
    y_m &= h_m(x_m),
\end{align}
which incorporates any prior information available about the plant. It is assumed that the state ($x_m$ and $x_p$) for both systems belongs to $\R^n$, that the inputs and outputs for both systems belong to $\R^q$, and that each of the mappings in \eqref{eq:plant_dynamics} and \eqref{eq:model_dynamics} are smooth. We make the following assumption about the model and plant:

\begin{assumption}\vspace{0.1cm}
    The plant and model have the same well-defined relative degree $\gamma = (\gamma_1, \gamma_2, \dots, \gamma_q)$ on all of $\R^n$.\vspace{0.1cm}
\end{assumption}
\begin{assumption}
The model and plant are both exponentially minimum phase. \vspace{0.1cm}
\end{assumption}

With these assumptions in place, we know that there are globally-defined linearizing controllers for the plant and model, which respectively take the following form:
\begin{align*}
    u_p(x,v) = \beta_p(x) + \alpha_p(x)v \\
    u_m(x,v) = \beta_m(x) + \alpha_m(x)v
\end{align*}
While $u_m$ can be calculated using the model dynamics and the procedures outlined in the previous section, the terms comprising $u_p$ are unknown to us. However, we do know that they may be expressed as
\begin{align*}
\beta_p(x) = \beta_m(x) + \Delta b(x)\\
\alpha_p(x) = \alpha_m(x) + \Delta \alpha(x)
\end{align*}
where $\Delta \beta \colon \R^n \to \R^q$ and $\Delta \alpha \colon \R^n \to \R^{q\times q}$ are unknown but continuous functions. Thus we construct an estimate for $u_p$ of the form
\begin{equation*}
    \hat{u}(\theta, x,v) = \big(\beta_m(x) + \beta_{\theta_1}(x)\big) + \big(\alpha_m(x) + \alpha_{\theta_2}(x)\big) v
\end{equation*}
where $\beta_{\theta_1} \colon \R^n \to \R^q$ is a parameterized estimate for $\Delta \beta$, and $\alpha_{\theta_2} \colon \R^n \to \R^{q \times q}$ is a parameterized estimate for $\Delta \alpha$. The parameters $\theta_1 = (\theta_1^1, \theta_1^2, \dots, \theta_1^{K_1}) \in \R^{K_1}$ and $\theta_2 =(\theta_2^1, \theta_2^2, \dots, \theta_2^{K_2}) \in \R^{K_2}$ are to be learned during online operation of the plant, and the total set of parameters $\theta \in \R^{K_1 + K_2}$ are collected by stacking $\theta_1$ on top of $\theta_2$. Our theoretical results will assume that the estimates are of the form
\begin{align}\label{eq:linparam}
    \beta_{\theta_1}(x) = \sum_{k = 1}^{ K_1}\theta_1^k \beta_k(x) 
    \ \ \ \  \alpha_{\theta_2}(x)   = \sum_{k=1}^{K_2} \theta_2^k \alpha_k(x)
\end{align}
where $\set{\beta_k}_{k=1}^{K_1}$ and $\set{\alpha_k}_{k=1}^{K_2}$ are linearly independent bases of functions, such as polynomials or radial basis functions.

\subsection{Idealized continuous-time behavior}
We now introduce a continuous-time update rule for the parameters of the learned linearizing controller which assumes that we know the functional form of the nonlinearities of the system. In Section \ref{sec:dt_updates}, we demonstrate how to approximate this ideal behavior in the sampled data setting using a policy gradient update rule which requires no information about the structure of the plant's nonlinearities.

We begin by assuming that there exists a set of ``true'' parameters $\theta^* =(\theta_1^*, \theta_2^*) \in \R^{K_1 + K_2}$ for the plant so that for each $x \in \R^n$ and $v \in \R^q$ we have $\hat{u}(\theta^*,x,v) \equiv u_p(x,v)$. In this case, we can write our parameter estimation error as $\phi =(\theta_1-\theta_1^*, \theta_2-\theta_2^*)$ so that $\theta = \phi + \theta^*$.

With the gain matrix $K$ constructed as in Section \ref{subec:tracking}, an estimate for the feedback linearization-based tracking controller is of the form
\begin{equation}
    u = \hat{u}(\theta, x, y_d^\gamma + Ke).
\end{equation}
When this control law is applied to the system the closed-loop error dynamics take the form
\begin{equation}\label{eq:ct_error_dynamics}
    \dot{e} = (A + BK)e + B W(x, y_d^\gamma, e)\phi
\end{equation}
where $W$ is a complicated function of $x, y_d^\gamma$ and $e$ which contains terms involving $b_p(x), A_p(x), \beta_m(x), \alpha_m(x), \beta_p(x)$ and $\alpha_p(x)$. The exact form of this function can be found in the technical report. The term $BW\phi$ captures the effects that the parameter estimation error $\phi$ has on the closed loop error dynamics. As we have done here, we will frequently drop the arguments of $W$ to simplify notation. We will also write $W(t)$ for $W(x(t),y_d^{\gamma}(t),e(t))$ when we wish to emphasize the dependence of the function on time.

Ideally, we would like to drive $BW\phi \to 0$ as $t \to \infty$ so that we obtain the desired closed-loop error dynamics \eqref{eq:error_dyn}. Recalling from Section \ref{subsec:mimo} that the reference model is designed such that $B^T B = I$, this suggests applying the least-squares cost signal
\begin{equation}\label{eq:ct_reward}
    R(t) =\frac{1}{2} \|BW\phi\|_2^2 = \frac{1}{2}\| W\phi\|_2^2
\end{equation}
and following the negative gradient of the cost with the following update rule:
\begin{equation}\label{eq:ct_gradient}
    \dot{\phi} = -W^TW\phi.
\end{equation}
Least-squares gradient-following algorithms of this sort are well studied in the adaptive control literature \cite[Chapter 2]{sastry1989adaptive}. Since we have $\dot{\theta} = \dot{\phi}$, this suggests that the parameters should also be updated according to $\dot{\theta} = -W^TW\phi$. Altogether, we can represent the tracking and parameter error dynamics with the linear time-varying system
\begin{equation}\label{eq:total_ct_dynamics}
    \begin{bmatrix}
    \dot{e} \\ \dot{\phi}
    \end{bmatrix}= \underbrace{\begin{bmatrix}A +BK & BW(t) \\ 0 & - W^T(t)W(t)\end{bmatrix}}_{A(t)}\begin{bmatrix}
    e \\ \phi
    \end{bmatrix}.
\end{equation}
Letting $X = (e^T, \phi^T)^T$, the solution to this system is given by
\begin{equation}
    X(t) = \Phi(t,0)X(0)
\end{equation}
where for each $t_1, t_2 \in \R^n$  the state transition matrix $\Phi(t_1,t_2)$ is the solution to the matrix differential equation $\frac{d}{dt}\Phi(t,t_2) = A(t) \Phi(t,t_2)$ with intial condition $\Phi(t_2,t_2) = I$, where $I$ is the identity matrix of appropriate dimension. From the adaptive control literature, it is well known that if $W(t)^T W(T)$ is ``persistently exciting'' in the sense that there exists $\delta > 0$ such that for each $t_0 \geq 0$
\begin{equation}\label{eq:pe}
    c_1 I > \int_{t_0}^{t_0 + \delta} W^T(t) W(t) dt > c_2 I
\end{equation}
for some $c_1, c_2>0$, then the time varying system \eqref{eq:total_ct_dynamics} is exponentially stable, if $W(t)$ also remains bounded. Intuitively, this condition simply ensures that the regressor term $W^TW$ is ``rich enough'' during the learning process to drive $\phi \to 0$ exponentially quickly. Observing \eqref{eq:ct_error_dynamics} we also see  that if $\phi \to 0$ exponentially quickly then $e \to 0$ exponentially as well.
We formalize this point with the following Lemma:
\begin{lemma}\label{lemma:pe}\vspace{0.1cm}
Let the persistence of excitation condition \eqref{eq:pe} hold and assume that there exists $C>0$ such that $\norm{W(t)}<C$ for each $t \in \R$. Then there exists $M>0$ and $\zeta>0$ such that for each $t_1, t_2 \in \R$
\begin{equation}
    \|\Phi(t_1,t_2)\| \leq Me^{-\zeta(t_1 -t_2)} 
\end{equation}
with $\Phi(t_1,t_2)$ defined as above. \vspace{0.1cm} 
\end{lemma}

Proof of this result can be found in Appendix, but variations of this result can be found in standard adaptive control texts \cite{sastry1989adaptive}. Unfortunately, we do not know the terms in \eqref{eq:ct_gradient} since we don't know $\phi$ or $W$ so this update rule cannot be directly implemented. In the next section we introduce a model-free update rule for the parameters of the learned controller which approximates the continuous update \eqref{eq:ct_gradient} without requiring direct knowledge of $W$ or $\phi$.

\subsection{Sampled-data parameter updates with policy gradients}\label{sec:dt_updates}
Hereafter, we will assume that the control supplied to the plant can only be updated every $\Delta t$ seconds. While this setting provides a more realistic model for many robotic systems, sampling has the unfortunate effect of destroying the affine relationship between the plant's inputs and outputs \cite{grizzle1988feedback} which was key to the continuous-time design techniques discussed above. Nevertheless, we now introduce a framework for approximately matching the ideal tracking and parameter error dynamics introduced in the previous section in the sampled-data setting using an Euler discretization of the continuous-time reward \eqref{eq:ct_reward} and a policy-gradient based parameter update rule.

Before introducing our sampled-data control law and adaptation scheme, we first fix notation and discuss a few key assumptions our analysis will employ. To begin we let $t_k = k \Delta t$ for each $k \in \N$ denote the sampling times for the system. Letting $x(\cdot)$ denote the trajectory of the plant, we let $x_{k} =x(t_k) \in \R^n$ denote the state of the plant at the $k$-th sample. Similarly, we let $\xi(\cdot)$ denote the trajectory of the outputs and their derivatives as in \eqref{eq:zero_dynamics}, and we set $\xi_k = \xi(t_k)\in \R^{|\gamma|}$ (not to be confused with the $k$-th entry of $\xi$). Next we let $u_k \in \R^m$ denote the input applied to the plant on the interval $[t_k, t_{k+1})$. The parameters for our learned controller will be updated only at the sampling times, and we let $\theta_k \in \R^K$ denote the value of the parameters on $[t_k, t_{k+1})$. We again let $y_d(\cdot)$, $\xi_{d}(\cdot)$ and $y_d^{(\gamma)}(\cdot)$ denote the desired trajectory for the outputs and their appropriate derivatives, and let $\xi_{d,k} =\xi_d(t_k) \in \R^{|\gamma|}$ and $y^{(\gamma)}_{d,k} = y^{(\gamma)}_d(t_k) \in \R^q$, and $e_k = (\xi_{k} - \xi_{d,k})  \in \R^{|\gamma|}$. We make the following assumption about the desired output signals and their derivatives:

\begin{assumption}\label{ass:bounded_reference} \vspace{0.1cm}
    The signal $y_d(\cdot)$ is continuous and uniformly bounded. Furthermore, for each $j = 1, \dots, q$ the derivatives $\{\dot{y}_{j,d}(\cdot), \ddot{y}_{j,d}(\cdot), \dots , y_{j,d}^{(\gamma_j)}(\cdot)\}$ are also continuous and uniformly bounded. 
\end{assumption}
\begin{remark} \vspace{0.1cm}
Typical convergence proofs in the continuous-time adaptive control literature generally only require that $(y_{j,d}(\cdot), \dot{y}_{1,d}(\cdot), \dots , y_{j,d}^{\gamma_{j}-1}(\cdot))$ be continuous and bounded, but these methods also assume that the input to the plant can be updated continuously. In the sampled data setting, we require the continuity of $y_{j,d}^{\gamma_j}(\cdot)$ to ensure that it does not vary too much within a given sampling period. 
\end{remark}

After sampling the discrete-time tracking error dynamics obey a difference equation of the form
\begin{equation}
    e_{k+1} =  H_k(x_k,e_k,u_k)
\end{equation}
where $H_k \colon \R^n \times \R^{|\gamma|} \times \R^{q} \to \R^{|\gamma|}$ is obtained by integrating the dynamics of the nonlinear system and reference trajectory over $[t_k, t_{k+1})$. Generally, $H_k$ will no longer be affine in the input. However, the relationship is approximately affine for small values of $\Delta t$. Indeed, with Assumptions \ref{ass:bounded_reference} and \ref{ass:bounded_trajectories} in place, if we apply the control law
\begin{equation}\label{eq:dt_controller}
    u_k = u(\theta_k, x_k, y_{d,k}^{\gamma} + Ke_k),
\end{equation}
then an Euler discretization of the continuous time error dynamics \eqref{eq:ct_error_dynamics} yields 
\begin{equation}
    e_{k+1} = e_k + \Delta t(A+BK)e_k + \Delta t BW_k\phi_k + O(\Delta t^2)
\end{equation}
where we have set $W_k = W(x_k, \xi_k, y^{\gamma}_{d,k} + Ke_k)$.
Thus,  letting $\bar{A} = (I + \Delta t(A + BK))$, for small $\Delta t >0$ the continuous-time cost is well approximated by
\begin{equation}\label{eq:dt_reward}
    R(t_k) = \frac{1}{2}\|W_k \phi_k \|_2^2 \approx \frac{1}{2}\left\|  \frac{e_{k+1}- \bar{A}e_k}{\Delta t} \right\|_{2}^{2} \colon = R_k(x_k, e_k,u_k),
\end{equation}
where we note that $e_k$ and $e_{k+1}$ are both quantities which can be measured by numerically differentiating the outputs from the plant. Intuitively, the sampled-data cost $R_k$ provides a measure for how well the control $u_k$ matches the desired change in the tracking error  \eqref{eq:ct_error_dynamics} over the interval $[t_k, t_{k+1})$.

Next, we add probing noise to the control law \eqref{eq:dt_controller} to ensure that the input is sufficiently exciting and to enable the use of policy-gradient methods for estimating the gradient of the discrete-time cost signal. In particular, we will draw the input according as $u_k \sim \pi_k(\cdot |\theta_k, x_k, e_k)$, where 
\begin{equation}
\pi_k(\cdot | \theta_k , x_k, e_k) = \hat{u}\bigg(\theta_k, x_k, y_{d,k}^{\gamma} + K(\xi_{d,k} - \xi_k)\bigg) +\mathcal{W}_k
\end{equation}
and $\mathcal{W}_k = \mathcal{N}(0, \sigma ^2I)$ is additive zero-mean Gaussian noise. Methods for selecting the variance-scaling term $\sigma^2$ will be discussed below, however for now it is sufficient to assume that $\sigma^2$ is bounded. 

With the addition of the random noise we now define
\begin{equation}
    J_k(\theta_k) = \mathbb{E}_{u_k \sim \pi_k(\theta_k, x_k, e_k)} R_k(x_k, e_k, u_k),
\end{equation}
noting that it is also common for policy gradient methods to use an expected ``cost-to-go'' as the objective.
Regardless, using the policy-gradient theorem \cite{sutton2000policy}, the gradient of $J_k$ can be written as
\begin{equation}
    \nabla_{\theta_k}J_k(\theta_k) = \mathbb{E}_{\pi_k}R_k(x_k,\xi_k, u_k)\,\cdot  \nabla_{\theta_k} \log \mathbb{P}\{\pi_k(u_k |\theta_k,x_k, e_k)\}
\end{equation}
where the expectation accounts for randomness due to the input $u_k = \pi_k(u_k |\theta_k, x_k, e_k)$. 

Moreover, a noisy, unbiased estimate of $\nabla J_k$ is given by
\begin{equation}\label{eq:grad_estimate}
    \hat{J}_k = R_k(x_k, \xi_k, u_k) \nabla_{\theta_k}\log\big(\mathbb{P}\{\pi(u_k |\theta_k, \theta_k, x_k, e_k)\}\big)
\end{equation}
where $u_k = \pi_k$ and is the actual input applied to the plant over the k-th time interval. Recall that $R_k(x_k, e_k,u_k)$ can be directly calculated using $e_k$, $e_{k+1}$ and \eqref{eq:dt_reward}, and $\nabla_{\theta_k}\mathbb{P}\{\log(\pi(u_k |\theta_k, s_k))\}$ can also be computed since the derivatives of $\hat{u}$ (and thus of $\log \mathbb{P}\{\pi_k\}$) are known to us. Thus, $\hat{J}_k$ can be computed using values that we have assumed we can measure. However, since the input $u_k$ is random, the gradient estimate is drawn according to 
\begin{equation}
    \hat{J}_k \sim \Delta \hat{J}_k(\cdot|  \theta_k, x_k, e_k)
\end{equation}
where the random variable is constructed using the relationship \eqref{eq:grad_estimate}. Using our estimate of the gradient for the discrete-time reward we propose the following noisy update rule for the parameters of our learned controller:
\begin{equation}
    \theta_{k+1} = \theta_k - \Delta t \hat{J}_k
\end{equation}
Putting it all together, the sampled-data stochastic version of our error dynamics becomes
\begin{align}\label{eq:total_dt_dynamics}
    e_{k+1} &= e_k + H_k(x_k, e_k,u_k)\\\nonumber 
    \phi_{k+1} &= \phi_k - \Delta t \hat{J}_k&
\end{align}
where $u_k = \pi_k$ and $\hat{J}_k$ is calculated as in \eqref{eq:grad_estimate}. We make the following Assumptions about this stochastic process: 

\begin{assumption}\label{ass:bounded_noise}\vspace{0.1cm}
There exists a constant $C>0$ such that $\sup_{k\geq0}\| w_k\| < C$ almost surely. \vspace{0.1 cm}
\end{assumption}

\begin{assumption}\label{ass:bounded_trajectories}\vspace{0.1cm}
There exists a constant $C>0$ such $\sup_{k \geq 0} \|x_k \| < C$ and $\sup_{k \geq 0} \|\theta_k \| < C$ almost surely.\vspace{0.1cm}
\end{assumption}

Assumption \ref{ass:bounded_noise} ensures that the additive noise does not drive the state to be unbounded during a single sampling interval, while Assumption \ref{ass:bounded_trajectories} ensures that the gradient estimate does not become undefined during the learning process. These important technical assumptions are common in the theory of stochastic approximations \cite{borkar2009stochastic}, and allow us to characterize the estimator for the gradient as follows:

\begin{lemma}\vspace{0.1cm}\label{lemma:grad_estimate}
Let Assumptions \ref{ass:bounded_reference}-\ref{ass:bounded_trajectories} hold. Then $\Delta \hat{J}_k(\cdot | \theta_k, x_k, e_k)$ is a sub-Gaussian distribution where
\begin{equation}\label{eq:grad_error}
    \mathbb{E}[\Delta \hat{J}_k(\cdot | \theta_k, x_k, e_k)] = W_k^TW_k\phi_{k} + O(\Delta t(1 + \sigma + \sigma^2))
\end{equation}
and 
\begin{equation}
    \norm{\Delta \hat{J}_k(\cdot | \theta_k, x_k, e_k)}_{\psi_2}= O\left(\frac{1}{\sigma}\right).
\end{equation}\vspace{0.1cm}
\end{lemma}

The Lemma demonstrates a trade-off between the bias and variance of the gradient estimate that has been observed in the reinforcement learning literature \cite{silver2014deterministic,zhao2011analysis}. Specifically, the bias of the gradient estimate decreases as $\sigma^2 \to 0$ but this causes the gradient of the estimator to blow up, as indicated by the increasing sub-Gaussian norm. However, the bias of the gradient estimate has a term which is $O(\Delta t)$ which does not depend on the amount of noise added to the system. This term comes from the fact that we have resorted to using a finite difference approximation \eqref{eq:dt_reward} to approximate the gradient of the continuous-time reward in the sampled data setting. Due to this inherent bias, little is gained by decreasing $\sigma^2$ past the point where $\sigma^2 = O(\Delta t)$. Next, we analyze the overall behavior of \eqref{eq:total_dt_dynamics}.

\subsection{Convergence analysis}\label{subsec:convergence}
The main idea behind our analysis is to model our sampled-data error dynamics \eqref{eq:total_dt_dynamics} as a perturbation to the idealized continuous-time error dynamics \eqref{eq:total_ct_dynamics}, as is commonly done in the stochastic approximation literature \cite{borkar2009stochastic}. Under the assumption that $W^TW$ is persistently exciting, the nominal continuous time dynamics are exponentially stable and we observe that the total perturbation accumulated over each sampling interval decays exponentially as time goes on. Due to space constraints, we outline the main points of the analysis here but leave the details to the technichal report. 

 Our analysis makes use of the piecewise-linear curve $\bar{\phi} \colon \R \to \R^{K}$ which is constructed by interpolating between $\phi_k$ and $\phi_{k+1}$ along the interval $[t_k, t_{k+1})$. That is, we define
\begin{equation*}
    \bar{\phi}(t) = \left( \frac{t_{k+1} -t}{\Delta t} \right) \phi_k + \left(\frac{t -t_{k}}{\Delta t} \right)\phi_{k+1} \ \text{ if } t \in [t_k, t_{k+1}).
\end{equation*}
Combining the tracking and interpolated tracking error into the state $X = (e^T, \phi^T)^T$ we may write
\begin{equation}\label{eq:perturbed_dynamics}
\frac{d}{dt}X(t) = A(t)X(t) + \delta(t)
\end{equation}
where  for each $t \in \R$ the dynamics matrix $A(t)$ constructed as in \eqref{eq:total_ct_dynamics} and the disturbance $\delta \colon \R\to \R^{|\gamma| + K}$ captures the deviation from the idealized continuous dynamics caused at each instance of time due the sampling, additive noise, and the process of interpolating the parameter error. Again letting $\Phi(t,\tau)$ denote the solution to $\frac{d}{dt}\Phi(t,\tau) =A(t) \Phi(t,\tau)$ with initial condition $\Phi(s,s) = I$, for each $t,s \in \R$ we have that
\begin{equation}
    X(t) = \Phi(t,0)X(0) + \int_{0}^{t} \Phi(t,\tau)\delta(\tau) d\tau
\end{equation}
Now, if we let $X_k = X(t_k)$ for each $k \in \mathbb{N}$ we can instead write
\begin{equation} \label{eq:dt_disturbance}
    X_k = \Phi(t_k,0)X_0 + \sum_{i=1}^{k-1}\Phi(t_k,t_{i+1})\underbrace{\int_{t_i}^{t_{i+1}}{\Phi(t_{i+1}, \tau)\delta(\tau) d\tau}}_{\delta_k},
\end{equation}
where the term $\delta_k \in \R^{|\gamma| + K}$ is the total disturbance accumulated over the interval $[t_k, t_{k+1})$. We separate the effects the distubance has on the tracking and error dynamics by letting $\delta_k^e \in \R^{|\gamma|}$ denote the first $|\gamma|$ elements of $\delta_k$ and letting $\delta_k^\phi \in \R^K$ denote the remaining entries.  
On the interval $[t_k, t_{k+1})$ the disturbance $\delta(t)$ can be written as a function of $u_k$, $x_k$ and $e_k$. Since $u_k$ is a random function of $x_k$, for fixed $x_k$, $e_k$ and $\theta_k$, the two elements of $\delta_k$ are distributed according to
\begin{equation}
    \delta_{k}^e \sim \Delta_k^e(\cdot|\theta_k, x_k, e_k) \ \  \text{ and } \ \ \delta_{k}^\phi \sim \Delta_k^\phi(\cdot|\theta_k. x_k, e_k).
\end{equation}
These random variables are constructed by integrating the distrubance over $[t_k,t_{k+1})$ and an explicit representation of these variable can be found in the proof of the following Lemma, which can be found in the Appendix. 

\begin{lemma} \label{lemma:noise_characterization}
Let Assumptions \ref{ass:bounded_reference}-\ref{ass:bounded_trajectories} hold. Then $ \Delta_{k}^e(\cdot|\theta_k, x_k, e_k)$ and $ \Delta_{k}^\phi(\cdot|\theta_k, x_k, e_k)$ are sub-Gaussian random variables where
\begin{equation}\label{lemma:total_disturbance}
   \| E[ \Delta_{k}^e(\cdot|\theta_k, x_k, e_k) ]\|_2 =O(\Delta t^2(1 + \sigma+ \sigma^2) ) \vspace{-1.5em}
\end{equation}
\begin{equation}
\| E[ \Delta_{k}^\phi(\cdot|\theta_k, x_k, e_k) ]\|_2 =O(\Delta t^2(1 + \sigma+ \sigma^2)) \vspace{-1em}
\end{equation}
\begin{equation}
    \|\Delta_k^e(\cdot |\theta_k, x_k,e_k)\|_{\psi_2} = O(\Delta t \sigma) \vspace{-1em}
\end{equation}
\begin{equation}
    \|\Delta_k^\phi(\cdot |\theta_k, x_k,e_k)\|_{\psi_2} = O\left(\frac{\Delta t }{\sigma}\right).
\end{equation}
\vspace{0.5em}
\end{lemma}

Next, for each $k \in \mathbb{N}$ we put  $\epsilon_k^e = E[ \Delta_{k}^e(\cdot|\theta_k, x_k, e_k) ] \in \R^{|\gamma|}$, ${\epsilon_k^\phi = E[ \Delta_{k}^\phi(\cdot|\theta_k, x_k, e_k) ]} \in \R^{K}$ and then define the zero-mean random variables $\mathcal{M}_k^e = \Delta_{k}^e(\cdot|\theta_k, x_k, e_k) - \epsilon_k^e$ and $\mathcal{M}_k^\phi = \Delta_{k}^\phi(\cdot|\theta_k, x_k, e_k) - \epsilon_k^\phi$. Our overall discrete-time process can then be written as 
\begin{equation}
    X_k = \Phi(t_k,0)X_0  +  \sum_{i=0}^{k-1}\Phi(t_k,t_{i+1})(\epsilon_i + \mathcal{M}_i).
\end{equation}
where $\epsilon_k \in \R^{|\gamma| + K}$ is constructed by stacking $\epsilon_k^e$ on top of $\epsilon_k^\phi$ and $\mathcal{M}_k$ is constructed by stacking $\mathcal{M}_k^e$ on top of $\mathcal{M}_k^\phi$. Now if we assume that $W^TW$ is persistently exciting, then for each $k_1, k_2 \in \mathbb{N}$ we have
\begin{equation}
    \norm{\Phi(t_{k_1}, t_{k_2})}  \leq Me^{-\zeta \Delta t(k_1 -k_2)} = M\rho^{k_1 -k_2}
\end{equation}
where $M>0$ and $\zeta>0$ are as in Lemma \ref{lemma:pe} and we have put $\rho = e^{-\zeta \Delta t}<1$. Thus, under this assumption we may use the triangle inequality to bound
\begin{equation}\label{eq:exp_decay}
    |X_k| \leq M\left( \rho^k |X_0| + \sum_{i=0}^{k-1}\rho^{k-i}|\epsilon_k| +  |\sum_{i=0}^{k-1}\rho^{k-i}\mathcal{M}_k |\right).
\end{equation}
Thus, when $W^TW$ is persistently exciting we see that the effects of the disturbance accumulated at each time step decays exponentially as time goes on, along with the effects of the initial tracking and parameter error. A full proof for the following Theorem is given in the Appendix, but the main idea is to use properties of geometric series to bound $ \sum_{i=0}^{k-1}\rho^{k-i}|\epsilon_k|$ over time and to use the concentration inequality from \cite[Theorem 2.6.3]{vershynin2018high} to bound the deviation of $|\sum_{i=0}^{k-1}\rho^{k-i}\mathcal{M}_k |$.

\begin{theorem}\label{thm:convergence}\vspace{0.1cm}
Let Assumptions \ref{ass:bounded_reference}-\ref{ass:bounded_trajectories} hold. Further assume that $W^T W$ is persistently exciting and let $M>0$ and $\zeta>0$ be defined as in Lemma \ref{lemma:pe}. Then there exists numerical constants $C_1>0$ and $C_2>0$ such that 
\begin{equation}\label{eq:total_bias}
   |\mathbb{E}[X_k]| \leq M\rho^k | X_0| + M C_1\frac{\Delta t(1 + \sigma + \sigma^2)}{\zeta}
\end{equation}
and for each $\lambda>0$ with probability $1-\lambda$ we have
\begin{equation}\label{eq:total_var}
    |X_k - \mathbb{E}[X_k]| \leq C_2 M\sqrt{\frac{ \Delta t \ln\left( \frac{2}{\lambda}\right)}{\zeta \sigma^2}}
\end{equation}\vspace{0.1cm}
\end{theorem}

Despite the high variance of the simple policy gradient parameter update analyzed so far, the Theorem demonstrates that with high probability our tracking and parameter errors concentrate around the origin. As $\Delta t$ decreases, the bias introduced by the sampling and additive noise diminish, as does the radius of our high-probability bound. These bounds also become tighter as the exponential rate of decay for the idealized continuous time dynamics increases. The Theorem again displays the trade-off between the bias and variance of the learning scheme observed in Section \ref{lemma:noise_characterization}. However, here we still observe in equation \eqref{eq:total_bias} that the bias introduced by the noise is relatively small, meaning $\sigma^2$ does not have to be made prohibitively small so as to degrade the bound in \eqref{eq:total_var}. 

\subsection{Variance Reduction via Baslines}\label{subsec: discussion}
It is common for policy gradients to be implemented with a \emph{baseline} \cite{williams1992simple}. In this case, the gradient estimator in \eqref{eq:grad_estimate} may become biased, though it often has lower variance \cite{sutton2018reinforcement, greensmith2004variance}. The expression with a baseline is
\begin{equation}
    \label{eq:grad_estimate_baseline}
    \hat{J}_k = \big(R_k(x_k, e_k, u_k) - S_k(x_k, e_k,u_k)\big)\,\cdot \nabla_{\theta_k}\log\big(\mathbb{P}\{\pi(u_k |\theta_k, \theta_k, x_k, e_k)\}\big),
\end{equation}
where $S_k(x_k, \xi_k, u_k)$ is an estimate of $R(x_k, \xi_k, u_k)$. If $S_k$ does not depend on $u_k$ then the addition of the baseline does not add any bias to the gradient estimate \cite{sutton2018reinforcement}. For example, in our numerical example below we use a simple sum-of-past-rewards baseline by setting $S_k = \sum_{i=0}^{k-1}R_i$, where $R_i$ is the $i$-th reward recorded. We consider it a matter of future work to rigorously study the effects of this an other common baselines from the reinforcement learned literature within the theoretical framework we have developed.


%% file: examples.tex
Our numerical example examines the application of our method to the double pendulum depicted in Figure~\ref{fig:dp} (a), whose dynamics can be found in \cite{shinbrot1992chaos}. With a slight abuse of notation, the system has generalized coordinates $q =(\theta_1, \theta_2)$ which represent the angles the two arms make with the vertical. Letting $x = (x_1,x_2,x_3, x_4)=  (q,\dot{q})$, the system can be represented with a state-space model of the form \eqref{eq:nonlinear_sys}
where the angles of the two joints are chosen as outputs. It can be shown that the vector relative degree is $(2,2)$, so the system can be completely linearized by state feedback.  

\begin{figure*}[h!]\vspace{0.1cm}
\centering
\includegraphics[width=0.9\textwidth, trim=0.1cm 0cm 0.1cm 0cm, clip=true]{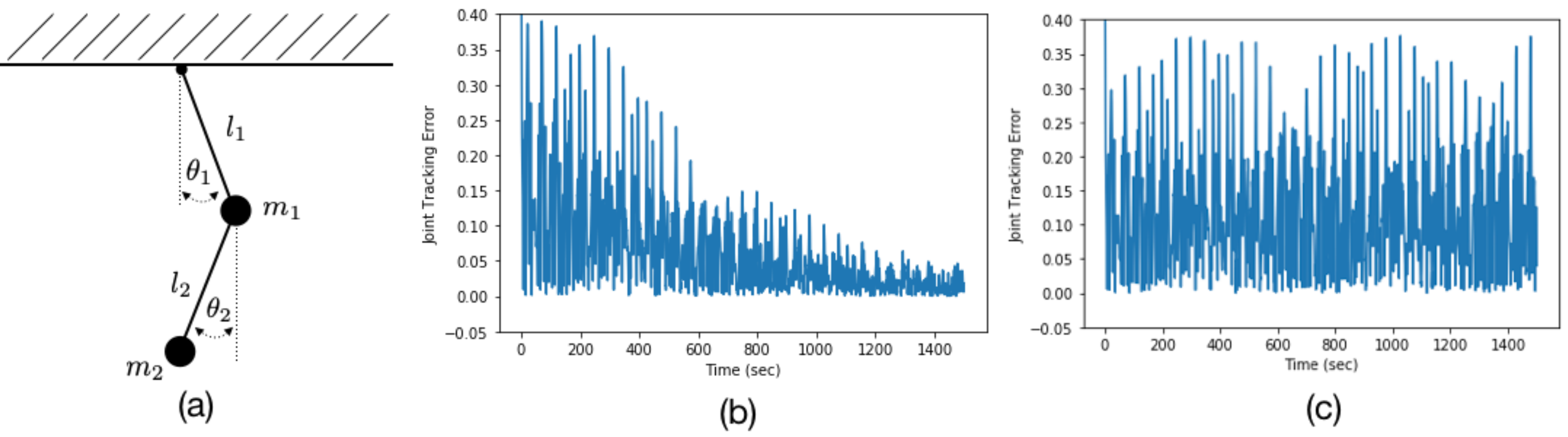}
\caption{(a) Schematic representation of the double pendulum model used in the simulations study. (b) The norm of the tracking error for the adaptive learning scheme (c) The tracking error for the nominal model-based controller with no learning. }
\label{fig:dp}
\vspace{0.1cm}
\end{figure*}

The dynamics of the system depend on the parameters $m_1$, $m_2$, $l_1$, $l_2$ where $m_i$ is the mass of the $i$-th link and $l_i$ its length. For the purposes of our simulation, we set the true parameters for the plant to be $m_1 = m_2 = l_1=l_2 = 1$. However, to set-up the learning problem, we assume that we have inaccurate measurements for each of these parameters, namely, $\hat{m}_1 = \hat{m}_2 = \hat{l}_1 = \hat{l}_2 = 1.3$. That is, each estimated parameter is scales to $1.3$ times its true value. Our nominal model-based linearizing controller $u_m$ is constructed by computing the linearizing controller for the dynamics model which corresponds to the inaccurate parameter estimates. The learned component of the controller is then constructed by using radial basis functions to populate the entries of $\set{\beta_k}_{k=1}^{K_1}$ and $\set{\alpha_k}_{k=1}^{k_2}$. In total, 250 radial basis functions were used.  

For the online leaning problem we set the sampling interval to be $\Delta t = 0.05$ seconds and set the level of probing noise at $\sigma^2 =0.1$. The reward was regularized using an average sum-of-rewards baseline as described in \ref{subsec: discussion}. The reference trajectory for each of the output channels were constructed by summing together sinusoidal functions whose frequencies are non-integer multiples of each other to ensure that the entire region of operation was explored. The feedback gain matrix $K \in \R^{2 \times 4}$ was designed so that each of the eigenvalues of $(A+BK)$ are equal to $-1.5$, where $A\in \R^{4 \times 4}$ and $B \in \R^{4 \times 2}$ are the appropriate matricies in the reference model for the system. 

Figure~\ref{fig:dp} (b) shows the norm of the tracking error of the learning scheme over time while Figure~\ref{fig:dp} (c) shows the norm of the tracking error for the nominal model-based controller with no learning. Note that the learning-based approach is able to steadily reduce the tracking error over time while keeping the system stable.

%% file: conclusion.tex
This paper developed an adaptive framework which employs model-free policy-gradient parameter update rules to construct a feedback-linearization based tracking controller for systems with unknown dynamics. We combined analysis techniques from the adaptive control literature and theory of stochastic approximations to provide high-confidence tracking guarantees for the closed loops system, and demonstrated the utility of the framework through a simulation experiment. Beyond the immediate utility of the proposed framework, we believe the analysis tools we developed provide a foundation for studying the use of reinforcement learning algorithms for online adaptation.

%% file: proofs.tex
The following Appedicies contain items which were too long to present in the main body of the document.  Appendix A containts two auxiliary Lemmas which are used extensively throughout the main proofs of Lemma \ref{lemma:grad_estimate} in Appendix B,  Lemma \ref{lemma:noise_characterization} in Appendix C, and Theorem \ref{thm:convergence} in Appendix $D$. Appendix E introduces the explicit form of the error equations in equation \eqref{eq:ct_error_dynamics}, and finally Appedix F provides proof for Lemma \ref{lemma:pe}. 
\subsection{Auxiliary Lemmas}

\begin{lemma}\label{lemma:tot_bound}
\ref{ass:bounded_reference}-\ref{ass:bounded_trajectories} hold. Then there exists a constant $C>0$ such that 
\begin{multline}\label{eq:boundy_1}
\sup_{t\in [t_k, t_{k})}\sup \bigg\{ \norm{A_p(x(t))}, \norm{B_p(x(t))}, \norm{f_p(x(t))}, \norm{g_p(x(t))}, \norm{\frac{d}{dx}A_p(x(t))}, \norm{\frac{d}{dx}B_p(x(t))}, \\ \norm{\frac{d}{dx}f_p(x(t))}, \norm{\frac{d}{dx}g_p(x(t))} \bigg \}< C 
\end{multline}
and
\begin{equation}\label{eq:boundy_2}
\norm{\hat{u}(\theta_k, x_k, e_k)} <C
\end{equation}\vspace{0.3cm}
\end{lemma}
\begin{proof}
The bound in \eqref{eq:boundy_1} follows directly from Assumption \ref{ass:bounded_trajectories} and the smoothness for the vector field for the plant. The bound in \eqref{eq:boundy_2} follows from Assumption \ref{ass:bounded_reference} and the continuity of the bases elements $\beta_k$ and $\alpha_k$.
\end{proof}\vspace{0.2cm}
\begin{lemma}\label{lemma:bounded_delta}
Let Assumptions \ref{ass:bounded_reference}-\ref{ass:bounded_trajectories} hold. Then there exits $C>0$ such that for each $t \in [t_k, t_{k+1})$ we have $\norm{\xi(t) -\xi_k} <C\Delta $ and $\norm{x(t) - x_k}<C\Delta t$ and $\norm{\bar{\phi}(t) - \phi_k}C_1 \Delta t.$.\vspace{0.3cm}
\end{lemma}
\begin{proof}
First, we have that $\dot{x} = f_p(x(t)) + g_p(x(t))u_k$ on the interval $[t_k, t_{k+1})$. By our standing Assumptions and the continuity of  $f_p$ and $g_p$ there exists a finite constant $C>0$ such that $\sup_{t \in [t_k, t_{k+1}) \norm{} f_p(x(t)) + g_p(x(t))u_k} <C$. This implies that $x(t) = x_k + \int_{t_k}^{t_{k+1}}  f_p(x(t)) + g_p(x(t))u_ \leq \Delta T C$, as desired. Noting that $\dot{\xi} = A\xi + B[b_p(x(t)) + A_p(x(t))u_k]$, the bound on $\norm{\xi(t) -\xi_k}$ follows by an analogous argument. To prove the bound for $\norm{\bar{\phi}(t) -\phi_k}$ we recall that $\bar{\phi}(t) = \hat{J}_{k}\frac{(t - t_k)}{\Delta t} + \phi_k$. However, the expression for $\hat{J}_k$ is given in equation \eqref{eq:expanded_grad} below, and we see is bounded under our standing Assumptions. Thus, there exists a $K>0$ such that $\norm{\hat{J}_k} < K$ and thus we have $\norm{\bar{\phi(t)} -\phi_k } \leq K \Delta t$. The desired result follows from the above observations. 
\end{proof}

\subsection{Proof of Lemma \ref{lemma:grad_estimate}}

Next, we note that we may rewrite
\begin{equation}
\mathbb{E}_{u_k \sim \pi_k(\cdot | \theta_k, x_k, e_k)} \left[ R_k(x_k, e_k, u_k) \right] =  \mathbb{E}_{w_k \sim \mathcal{W}_k} \left[ R_k(x_k, e_k, \hat{u}_{\theta_k} + w_k) \right] 
\end{equation}
where for convienience of notation we have defined
\begin{equation}
\hat{u}_{\theta_k} = \hat{u }(\theta_k, x_k, y_{d,k}^\gamma +Ke_k)
\end{equation}
and suppressed the dependence on $x_k, e_k$ and $y_{d,k}^\gamma$. Thus, we have that
\begin{equation}
\nabla_{\theta_k} J(\theta_k) =  \nabla_{\theta_k} \mathbb{E}_{w_k \sim \mathcal{W}_k} \left[ R_k(x_k, e_k, \hat{u}_{\theta_k} + w_k) \right]  = \mathbb{E}_{w_k \sim \mathcal{W}_k} \left[   \nabla_{\theta_k}R_k(x_k, e_k, \hat{u}_{\theta_k} + w_k) \right].
\end{equation}
Now, since $R_k = \norm{\frac{e_{k+1} - e_k}{\Delta t}}_2^2$ we can calculate $\nabla_{\theta_k}R_k(x_k, e_k, \hat{u}_{\theta_k} + w_k)$ by first  calculating how $e_{k+1}$ varies with $\theta_k$. Now, we have that
\begin{equation}
e_{k+1} = \xi_k - \xi_{d,k} + \int_{t_k}^{t_{k+1}} A \xi(t) + B\left [b_p(x(t)) + A_p(x(t)) \right( \hat{u}_{\theta_k} + w_k\left) \right]- \int_{t_k}^{t_{k+1}}  By_{d}^\gamma(t)dt
\end{equation}
and noting that $A\xi_k + B[b_p(x(t_k)) + A_p(x(t_k)) \hat{u}_{\theta_k}] - By_{d,k}^\gamma = B W_k \phi_k$ we can rewrite the above expression as
\begin{equation}\label{eq:error_decomp}
e_{k+1} = e_k + \Delta t(A + BK)e_k + \Delta t B W_k \phi_k  + \Delta t A_p(x_k)w_k + T_1+ T_2+T_3+ T_4 +T_5
\end{equation}
where we define the terms
\begin{align*}
T_1 &= \int_{t_k}^{t_{k+1}} A(\xi(t) - \xi_k) dt \\
T_2 &=\int_{t_k}^{t_{k+1}}B[b_p(x(t)) - b_p(x_k)]dt \\
T_3 & = \int_{t_k}^{t_{k+1}}B[A_p(x(t)) - A_p(x_k)]\hat{u}_{\theta_k} dt\\
T_4 & = \int_{t_k}^{t_{k+1}}B[A_p(x(t)) -A_p(x_k)]w_k dt\\
T_5 & = \int_{t_k}^{t_{k+1}} y_{d}^\gamma(t) - y_{d,k}^\gamma dt
\end{align*}
Thus, the reward may be rewritten as
\begin{align}\label{eq:expanded_reward}
R_k(x_k, e_k, \hat{u}_{\theta_k}) &=\frac{1}{2} \norm{\frac{e_{k+1} -e_{k}}{\Delta t} }_2^2= \frac{1}{2\Delta t^2} \left(\Delta t B W_k \phi_k+ \Delta t A_p(x_k)w_k  + \sum_{j=1}^5T_j \right) ^T \cdot \left(\Delta t B W_k \phi_k + \Delta t A_p(x_k)w_k+\sum_{j=1}^5T_j \right) \\
&= \frac{1}{2}\norm{W_k\phi_k}^2 +  \phi_k^TW_k^TB^T \cdot A_p(x_k)w_k  + \frac{1}{ \Delta t} \phi_k^TW_k^TB^T \cdot  \sum_{j=1}^5T_j   +  \norm{A_p(x_k) w_k}^2 \\ 
&+ w_k^TA_p(x_k) \cdot  \sum_{j=1}^5T_j  + \frac{1}{\Delta t^2}\sum_{i =1}^5 \sum_{j=1}^5 \left(T_i^T\right) T_j
\end{align}
The gradient of the reward with respect to the learned parameters is given by
\begin{align}\label{eq:expanded_reward2}
\nabla_{\theta_k} R_k(x_k, e_k, \hat{u}_{\theta_k} +w_k) &= W_k^TW_k\phi_k +  W_k^TB^T \cdot A_p(x_k)w_k+  \frac{1}{ \Delta t} W_kB^T  \cdot \sum_{j=1}^5T_j\\ \nonumber
 &+\frac{1}{\Delta t} w_k^TA_p(x_k)^T \cdot  \sum_{j=1}^5 \frac{\partial}{\partial_{\theta_k}}T_j +  \frac{1}{ \Delta t^2}  \sum_{i =1}^5 \sum_{j=1}^5\left( \frac{\partial}{\partial \theta_k}T_i^T\right) T_j 
\end{align} 
Now, $\mathbb{E}_{w_k \sim \mathcal{W}_k}[W_k^TW_k\phi_k] = W_k^TW_k\phi_k$ and $\mathbb{E}_{w_k \sim \mathcal{W}_k}W_k^TB^T \cdot A_p(x_k)w_k =0$. Thus, to obtain the desired bound we need only to bound the terms involving the $T_i$.

First, we will produce bounds for the $T_j$ terms. By Lemma \ref{lemma:bounded_delta} and the continuity of the vector field for the plant, we know that there exists $C_1>0$ such that for each $t \in [t_k, t_{k+1})$ we have $\norm{\xi(t) -\xi_k} \leq C_1 \Delta t$, $\norm{B_p(x(t)) -B_p(x_k)} \leq C_1 \Delta t$ and $\norm{A_p(x(t)) - A_p(x_k)}  \leq C_1\Delta t$. Furthermore, by Lemma \ref{lemma:tot_bound} there exists $C_2>0$ so that for each $t \in [t_k, t_{k+1})$ we have  $\norm{A_p(x(t))} \leq C_2$, $\norm{\frac{d}{dx}b_p(x(t))} \leq C_2$, $\norm{\frac{d}{dx}A_p(x(t))} \leq C_2$, $\norm{\hat{u}_{\theta_k}} \leq C_2$. Finally, by the continuity of $y_d^\gamma$ we know that for there exists $C_3 >0$ such that for each $t \in [t_k, t_{k+1})$ we have $\norm{y_d^\gamma(t) - y_{d,k}^\gamma} \leq C_3 \Delta t$. Putting these facts together we have
\begin{align}\label{eq:norm_bounds}
\norm{ T_1 }&= \norm{ \int_{t_k}^{t_{k+1}} A(\xi(t) - \xi_k) dt } \leq  \int_{t_k}^{t_{k+1}}\norm{A}C_1\Delta t dt = O(\Delta t^2) \\ \nonumber
\norm{T_2 }&=\norm{\int_{t_k}^{t_{k+1}}B[b_p(x(t)) - b_p(x_k)]dt}  \leq  \int_{t_k}^{t_{k+1}}\norm{B}C_1\Delta t dt =  O(\Delta t^2) \\  \nonumber
\norm{T_3} & =\norm{ \int_{t_k}^{t_{k+1}}B[A_p(x(t)) - A_p(x_k)]\hat{u}_{\theta_k} dt} \leq   \int_{t_k}^{t_{k+1}} \norm{B}C_1\Delta t  C_2dt = O(\Delta t^2) \\  \nonumber
\norm{T_4} & = \norm{\int_{t_k}^{t_{k+1}}B[A_p(x(t)) -A_p(x_k)]w_k dt }\leq  \int_{t_k}^{t_{k+1}}\norm{B}C_1 \Delta t \norm{w_k} = O(\Delta t ^2\norm{w_k})\\  \nonumber
\norm{T_5} & = \norm{\int_{t_k}^{t_{k+1}} y_{d}^\gamma(t) - y_{d,k}^\gamma dt} \leq \int_{t_k}^{t_{k+1}} C_3 \Delta dt = O(\Delta t^2).
\end{align}
Next, we bound the terms of the form $\frac{\partial }{\partial \theta_k} T_j$. Since the $T_j$ depend on $x(\cdot)$ and $\xi(\cdot)$, we first bounded how much these trajectories vary over the interval $[t_k, t_{k+1})$ as the learned parameter is changed.
By \cite[Theorem 5.6.2]{polak2012optimization}  and use of the chain rule, for each  $t \in [t_k, t_{k+1)}$ we have $\frac{\partial }{\partial \theta_k} x(t) = \Xi(t)$ where the sensitivity matrix $\Xi(t) \in \R^{n \times (K_1 + K_2)}$ is given for each $t \in [t_k, t_{k+1})$ by 
\begin{equation}\label{eq:sensitivity}
\Xi(t) = \int_{t_k}^{t_{k+1}}\exp\left(\int_{\tau}^{t_{k+1} }A(s)ds\right)B(\tau)  \frac{\partial}{\partial \theta_k}\hat{u}_{\theta_k}d\tau 
\end{equation}
where for each $t \in [t_k, t_{k+1})$ we have $A(t) \in \R^{n\times n}$, $B(t) \in \R^{n \times q}$ and $\frac{\partial}{\partial \theta_k}\hat{u}_{\theta_k} \in \R^{q \times (K_1 + K_2)}$ where 
\begin{equation}
A(t) =\frac{\partial }{\partial x}\bigg(f_p(x(t)) + g_p(x(t))(\hat{u}_{\theta_k} +w_k) \bigg)=  \frac{d}{dx}f_p(x(t)) + \sum_{i=1}^{q}\frac{d}{dx}g_{p,i}(x(t))^T(\hat{u}_{\theta_k} +w_k)^i \ \ \ B(t) = g_p(x(t))
\end{equation}
where $g_{p,i}(x)$ is the $i$-th column of $g_p(x)$ and $(\hat{u}_{\theta_k} +w_k)^i$ is the $i$-th entry of $(\hat{u}_{\theta_k} +w_k)$ and $\frac{\partial}{\partial \theta_k}\hat{u}_{\theta_k} = \frac{\partial}{\partial \theta_k}\hat{u}(\theta_k, x_k, y_{d,k}^\gamma + Ke_k)$. Now, by Assumptions \ref{ass:bounded_reference}-\ref{ass:bounded_trajectories} and the smoothness of $g_p$ and $f_p$ there exists $K_1>0$ and $K_2>0$ such that for each $t \in [t_k, t_{k+1})$ we have $\| A(t)\| \leq K_1$ and $\| B(t)\| \leq K_2$ for any choice of matrix norm. Furthermore, we have for each $i = 1, \dots , K_1$ and $j = 1, \dots, K_2$ we have 
\begin{equation}
\frac{\partial }{\partial \theta_{1}^i}\hat{u}_{\theta_k}= \beta_i(x_k)  \ \ \ \ \ \ \ \ \  \frac{\partial }{\partial \theta_{2}^j}= \alpha_j(x_k)(y_{d,k}^\gamma + Ke_k).
\end{equation}
Thus, by Assumptions \ref{ass:bounded_reference} and \ref{ass:bounded_trajectories} and the continuity of  the $\beta_k$ and $\alpha_k$ we then observe that there must exist $K_3>0$ such that $\norm{\frac{\partial}{\partial_{\theta_k}} \hat{u}_{\theta_k}} \leq K_3$. Using these facts and equation \eqref{eq:sensitivity} for each $t \in [t_k, t_{k+1})$ we have that 
\begin{equation}
\norm{\Xi(t)} \leq \int_{t_k}^{t_{k+1}} \norm{\exp\left(\int_{\tau}^{t_{k+1} }A(s)ds\right)} \cdot \norm{B(\tau)} \cdot \norm{\frac{\partial}{\partial_{\theta_k}} \hat{u}_{\theta_k}} \leq  K_4\Delta t
\end{equation}
where $K_4 = \exp(K_1\Delta t) \cdot K_2 \cdot K_3$. Bound to bound the change in $\xi(t)$ as $\delta_k$ is varied. As discussed in \ref{sec:FBL} there exists a diffeomorphism $S \colon x \to (\xi, \eta)$ which takes the state $x$ to the new coordinates $(\xi, \eta)$. In particular, let $\xi =s_1(x)$ denote the first part of this transformation which gives the values of the outputs and their derivatives. By the chain rule we have that $\frac{\partial}{\partial_{\theta_k}} \xi(t) = \nabla s_1(x(t)) \cdot \frac{\partial }{\partial_{\theta_k}}x(t)$. However, by Assumiption \ref{ass:bounded_trajectories} and the continuity of $\nabla s(x(t))$ we know that $\sup_{[t_k, t_{k+1})} \nabla s(x(t))$ is bounded. Thus, there exists $K_5>0$ such that for each $t \in [t_k, t_{k+1})$ we have $\norm{\frac{\partial}{\partial_{\theta_k}} \xi(t) } \leq K_4 \Delta t$.  

Next, we apply the above bounds to bound terms of the form $\norm{\frac{\partial}{\partial_{\theta_k}} T_j}$. In particular, we have 
\begin{align}\label{eq:norm_bounds2}
\norm{ \frac{\partial}{\partial_{\theta_k} }T_1 }&= \norm{ \int_{t_k}^{t_{k+1}} A\left(\frac{\partial}{\partial_{\theta_k}}\xi(t)\right) dt } \leq \int_{t_k}^{t_{k+1}} \norm{A} \norm{\frac{\partial}{\partial_{\theta_k}}\xi(t)} \leq  \int_{t_k}^{t_{k+1}}\norm{A}K_5 \Delta t dt = O(\Delta t^2) \\ \nonumber
\norm{ \frac{\partial}{\partial_{\theta_k} }T_2 }&=\norm{\int_{t_k}^{t_{k+1}}B \left(\frac{\partial}{\partial_{\theta_k}}b_p(x(t))\right)   }\leq  \int_{t_k}^{t_{k+1}}\norm{B} \norm{\frac{d}{dx}B_p(x(t)) } \norm{\frac{\partial}{\partial_{\theta_k}}x(t)}dt \\\nonumber
&\leq \int_{t_k}^{t_{k+1}}\norm{B} C_2 K_4\Delta t dt =  O(\Delta t^2) \\ \nonumber
\norm{\frac{\partial}{\partial_{\theta_k} } T_3} & =\norm{ \int_{t_k}^{t_{k+1}}B\left(\frac{\partial}{\partial_{\theta_k}} A_p(x(t))\right)\hat{u}_{\theta_k} + B[A_p(x(t)) -A_p(x_k)]\frac{\partial}{\partial_{\theta_k}}\hat{u}_{\theta_k}dt}\\ \nonumber
& \leq   \int_{t_k}^{t_{k+1}} \norm{B} \left( \norm{\frac{d}{dx}A_p(x(t))}\cdot \norm{\frac{\partial}{\partial_{\theta_k}}x(t)}  \norm{\hat{u}_{\theta_k}}+ \cdot \norm{ [A_p(x(t)) -A_p(x_k)]}\norm{\frac{\partial}{\partial_{\theta_k}}\hat{u}_{\theta_k} }\right)dt\\ \nonumber
&\leq  \int_{t_k}^{t_{k+1}}  \norm{B} \left( C_2\cdot K_3 \Delta t + C_1 \Delta t  K_3\right) dt = O(\Delta t^2) \\  \nonumber
\norm{ \frac{\partial}{\partial_{\theta_k} } T_4} & = \norm{\int_{t_k}^{t_{k+1}}BA_p(x(t))w_k dt }\leq  \int_{t_k}^{t_{k+1}}\norm{B} \cdot \norm{\frac{d}{dx}A_p(x(t))} \cdot \norm{\frac{\partial}{\partial_{\theta_k}}x(t)} \cdot \norm{w_k} \\ \nonumber
&\leq \int_{t_k}^{t_{k+1}}\norm{B}\cdot C_2\cdot K_3 \Delta t \cdot \norm{w_k}  =  O(\Delta t^2 \norm{w_k})\\ \nonumber
\norm{\frac{\partial}{\partial_{\theta_k} } T_5} & = 0.
\end{align}
Returning to our expression for $\nabla_{\theta_k}R_k$ in \eqref{eq:expanded_reward2} using the bound on terms of the form $\norm{T_j}$ from \eqref{eq:norm_bounds} we have that
\begin{equation}
\norm{\frac{1}{2 \Delta t} W_kB^T  \cdot \sum_{j=1}^5T_j } \leq \norm{\frac{1}{2 \Delta t} W_kB^T } \cdot \sum_{j=1}^5\norm{T_j }  =   \norm{\frac{1}{2 \Delta t} W_kB^T } O(\Delta t^2 (1 +\norm{w})) = O(\Delta t(1 +\norm{w_k})) \\
\end{equation}
and if we additionally use the bounds of the form $\norm{\frac{\partial}{\partial_{\theta_k}} T_i}$ from \eqref{eq:norm_bounds2} we may bound
\begin{equation}
\norm{\frac{1}{\Delta t} w_k^TA_p(x_k)^T \cdot  \sum_{j=1}^5 \frac{\partial}{\partial_{\theta_k}}T_j} \leq \frac{1}{\Delta t } \norm{A_p(x_k)}\norm{w_k} O(\Delta t^2(1 + \norm{w_k})) = O(\Delta t(\norm{w_k} + \norm{w_k}^2))  
\end{equation}
\begin{equation}
 \norm{\frac{1}{ \Delta t^2}  \sum_{i =1}^5 \sum_{j=1}^5\left( \frac{\partial}{\partial \theta_k}T_i^T\right) T_j } \leq \frac{1}{ \Delta t^2}   \sum_{i =1}^5 \sum_{j=1}^5 \norm{\frac{\partial}{\partial_{\theta_k}} T_j} \cdot \norm{T_i}  = O(\Delta t( 1 +  \norm{w_k} + \norm{w_k}^2))
\end{equation}
Putting together the above bounds with \eqref{eq:expanded_reward2} we have that 
\begin{equation}
\nabla_{\theta_k} R(x_k, e_k, \hat{u}_{\theta_k}+w_k) = W_k^TW_k\phi_k + O(\Delta t(1 + \norm{w_k} + \norm{w_k}^2))
\end{equation}
and
\begin{align*}
\mathbb{E}\hat{J}_k &= \mathbb{E}_{w_k \sim \mathcal{W}_k} [\nabla_{\theta_k} R(x_k, e_k, \hat{u}_{\theta_k} + w_k)] = \mathbb{E}_{w_k \sim \mathcal{W}_k}\left[ W_k^T W_k \phi_k + W_k^TB^T \cdot A_p(x_k)w_k \right]+ O(\Delta t(1 + \norm{w_k} + \norm{w_k}^2)) \\
&= W_k^T W_k \phi_k O(\Delta t(1 + \sigma + \sigma^2))
\end{align*}
where we have used the fact that $E_{w_k \sim \mathcal{W}_k}\norm{w_k} = \sigma$ and $E_{w_k \sim \mathcal{W}_k}\norm{w_k}^2 = \sigma^2$ to show the desired result for the bias of the gradient estimate. 

Next, we bound the sub-Gaussian norm of the estimator. We omit some details in the interest of brevity, but the main idea is to first show that the gradient estimate is a Lipschitz continuous function of $w_k$ where $w_k$ is the realization of $\mathcal{W}_k$. We then use Theorem 5.2.15 from  \cite[Theorem 2.6.3]{vershynin2018high} to demonstrate that $\Delta \hat{J}_k$ is a sub-Gaussian random variable. When specialized to our setting, the cited Theorem says that if $X \sim \mathcal{N}(\bar{\mu}, \bar{\sigma}^2I)$ is a Gussian random variable with finite mean $\bar{\mu}$ and $\sigma^2$, then the random variable $T(X)$ where $T$ is a Lipschitz continuous map is sub-Gaussian with norm $\norm{T(X)}_{\psi_2} \leq \frac CL\sigma$ where $C>$ is an absolute constant and $L>0$ is a Lipschitz constant for $T$. 

Next, letting $R_k = $ The estimate for the gradient can be expanded as
\begin{equation}\label{eq:expanded_grad}
\hat{J}_k = R_k\nabla_{\theta_k}\log\big(\mathbb{P}\{\pi(u_k |\theta_k, \theta_k, x_k, e_k)\} = R_k \nabla_{\theta_k} \left( -\frac{1}{2} \sum_{i=1}^{q}  \frac{(u_k^i - \hat{u}_k^i)^2}{\sigma^2} + 2 \log(\sigma) + \log(2\pi)\right) =  R_k\sum_{i=1}^{q} \left( \frac{u_k^i - \hat{u}_k^i}{\sigma^2}\frac{\partial }{\partial_{\theta_k}} \hat{u}_{\theta_k}^i\right)
\end{equation}
where $u_k^i$ is the $i$-th entry of $u_k$, $\hat{u}_k^i)$ is the $i$-th entry of $\hat{u}_k^i)$. Here, we have used the fact that $u_k \sim \pi_k(\cdot | \theta_k, x_k, e_k) = \mathcal{N}\left( \hat{u}_\theta, \sigma^2 I \right)$ and used the formula the logarithm of normal distributions. Noting the $u_k^i - \hat{u}_{\theta_k}^i = w_k$, the above expression can be rewritten as
\begin{equation}
\hat{J}_k = R_k\sum_{i=1}^{q} \frac{w_k^i}{\sigma^2}\frac{\partial }{\partial_{\theta_k}} \hat{u}_{\theta_k}^i = R_k \frac{1}{\sigma^2} w_k^T \frac{\partial }{\partial_{\theta_k}} \hat{u}_{\theta_k}.
\end{equation}
where $w_k^i$ is the $i$-th entry of the random variable $\mathcal{W}_k$. Now, by our preceding discussion and Assumptions \ref{ass:bounded_reference}-\ref{ass:bounded_trajectories} we see that $\norm{R_k} \leq \Lambda$ for some $ \Lambda >0$. Thus, we observe that $w_k \to R_k \frac{1}{\sigma^2} w_k^T \frac{\partial }{\partial_{\theta_k}} \hat{u}_{\theta_k}$.
is a Lipschitz continuous mapping of $w_k$ with Lipschitz constant $L =\frac{1}{\sigma^2} \cdot \Lambda \cdot \norm{ \frac{\partial }{\partial_{\theta_k}} \hat{u}_{\theta_k}}= O(\frac{1}{\sigma^2})$, where we have again use the fact that $\frac{\partial }{\partial_{\theta_k}} \hat{u}_{\theta_k} $ is bounded, as was established above. Thus, by  \cite[Theorem 2.6.3]{vershynin2018high} we see that $\nabla \hat{J}$ is a sub-Gaussian random variable with norm on the order of $O(\frac{1}{\sigma})$, as desired.  

\subsection{Proof of Lemma \ref{lemma:noise_characterization}}
We first demonstrate how to calculate $\delta_k^e$, the disturbance for the tracking error over the interval $[t_k, t_{k+1})$.
For convenience we re-write \eqref{eq:error_decomp}:
\begin{equation}\label{eq:error_decomp3}
e_{k+1} = e_k + \Delta t(A + BK)e_k + \Delta t B W_k \phi_k  + \Delta t A_p(x_k)w_k + \sum_{i =1}^{5}T_i
\end{equation}
Now, we may re-write
\begin{equation}
 \Delta t(A + BK)e_k + \Delta t B W_k \phi_k  = \int_{t_k}^{t_{k+1}} (A + BK)e(t) + BW(t)\bar{\phi}(t)dt  + \int_{t_k}^{t_{k+1}} (A+BK)(e_k - e(t))dt  +  \int_{t_k}^{t_{k+1}} B[ W_k\phi_k - W(t)\bar{\phi}(t)]dt
\end{equation}
Now, by Lemma \ref{lemma:bounded_delta} for each $t \in [t_k, t_{k+1})$ there exists $C_1>0$ such that we have $\norm{e(t) - e_k} \leq C_1 \Delta t $ and $\norm{\bar{\phi}(t) - \phi_k} \leq C_1 \Delta t$. Thus, we may bound
\begin{equation}
\norm{\int_{t_k}^{t_{k+1}} (A+BK)(e_k - e(t))dt} \leq \int_{t_k}^{t_{k+1}} \norm{A+BK}\norm{(e_k - e(t))} dt \leq \norm{A+BK}C_1 \Delta t dt = O(\Delta t^2)
\end{equation}
Furthermore, we may expand
\begin{equation}
W_k\phi_k - W(t)\bar{\phi}(t) = W_k(\phi_k-\bar{\phi(t)}) - (W(t) -W_k)\bar{\phi}(t)
\end{equation}
Thus we may further bound
\begin{equation}
\norm{W_k\phi_k - W(t)\bar{\phi}(t) }= \norm{W_k(\phi_k-\bar{\phi(t)}) } +  \norm{(W(t) -W_k)\bar{\phi}(t)} \leq \norm{W_k}\cdot \norm{\phi_k-\bar{\phi}(t))} + \norm{W(t) -W_k} \cdot \norm{\bar{\phi}(t)} = O(\Delta t )
\end{equation}
where we have used the fact that $\norm{\phi_k-\bar{\phi}(t))} \leq C_1 \Delta t$, the fact that $W_k$ is bounded by Lemma \ref{lemma:tot_bound} and the fact that $W(t)$ is a bounded continuous function of time by Assumptions \ref{ass:bounded_reference} and \ref{ass:bounded_trajectories} so that $\norm{W_k - W(t)} = O(\Delta t)$ for each $t \in [t_k, t_{k+1})$. This then implies that $\int_{t_k}^{t_{k+1}} = O(\Delta t^2)$. Combining these decompositions with \eqref{eq:error_decomp3} we have
\begin{equation}
e_{k+1} = e_k + \Delta t(A + BK)e_k + \Delta t B W_k \phi_k  + A_p(x_k)w_k + \sum_{i =1}^{7}T_i 
\end{equation}
where we have set
\begin{align}
T_6 &= \int_{t_k}^{t_{k+1}} (A+BK)(e_k - e(t))dt\\
T_7 & = \int_{t_k}^{t_{k+1}} B[ W_k\phi_k - W(t)\bar{\phi}(t)]dt
\end{align}
thus using the above bounds we have that
\begin{equation}\label{eq:state_dist}
 \delta_k^e =  \Delta t A_p(x_k)w_k+   \sum_{i =1}^{7}T_i = \Delta t A_p(x_k)w_k+ O(\Delta t^2(1 + \norm{w_k}))
\end{equation} 
 Furthermore, we have 
 \begin{equation}
 \mathbb{E}[\Delta_k^e] = \mathbb{E}_{w_k \sim \mathcal{W}_k} \delta_k = \mathbb{E}_{w \sim \mathcal{W}_k}[A_p(x_k)w_k] +  \mathbb{E}_{w \sim \mathcal{W}_k} [O(\Delta t(1 + \norm{w_k}))] = O(\Delta^2 t(1+\sigma )),
\end{equation} 
where we have used the fact that $\mathbb{E}_{w_k \sim \mathbb{W}_k} \| w_k\|= (\sigma)$. 

Next, we demonstrate how to calculate $\delta_k^\phi$, the disturbance to the parameter error over the interval $[t_k, t_{k+1})$. Now we have that 
\begin{align}
\phi_k = \bar{\phi}(t_k) = \phi_k + \int_{t_k}^{t_{k+1}} W(t)^TW(t) \bar{\phi}(t) dt + \int_{t_k}^{t_{k+1}} \hat{J}_k - W(t)^TW(t)\bar{\phi}(t)dt
\end{align}
thus we have 
\begin{equation}
\delta_k^\phi = \int_{t_k}^{t_{k+1}} \hat{J}_k - W(t)\bar{\phi}(t)dt
\end{equation}
Now, 
\begin{align}\label{eq:tracking_distrubance}
\mathbb{E}_{w_k \sim \mathcal{W}_k} \left[ \int_{t_k}^{t_{k+1}} \hat{J}_k - W(t)\bar{\phi}(t)dt \right]  & = \int_{t_k}^{t_{k+1}} \mathbb{E}_{w_k \sim \mathcal{W}_k}[\hat{J}_k] - W(t)^TW(t)\bar{\phi}(t)dt \\ \nonumber
 & =  \int_{t_k}^{t_{k+1}}  W_k^TW_k \phi_k - W(t)^TW(t)\bar{\phi}(t) dt + O(\Delta t^2(1+  \norm{w_k}  + \norm{w_k}^2))
\end{align}
where in the last equality we have used Lemma \ref{lemma:grad_estimate}. Now, we have
\begin{equation}
W_k^TW_k\phi_k - W(t)^TW(t)\bar{\phi}(t) = W_k^TW_k(\phi_k-\bar{\phi(t)}) - (W(t)^T W(t) -W_k^TW_k)\bar{\phi}(t)
\end{equation}
Using the same argument we used to bound $W_k\phi_k - W(t)\bar{\phi}(t) = O(\Delta t)$, it is not difficult to show that $W_k^TW_k\phi_k - W(t)^TW(t)\bar{\phi}(t) =O(\Delta t)$. Combining this with \eqref{eq:tracking_distrubance} we see that $\mathbb{E}[\delta_k^\phi] = \mathbb{E}{O(\Delta t^2(1+  \norm{w_k}  + \norm{w_k}^2))} = O(\Delta t^2(1 + \sigma +\sigma^2))$ as desired, where we have used the fact that $\mathbb{E}[\norm{w_k}] = O(\sigma)$ and $\mathbb{E}[\norm{w_k}] = O(\sigma^2)$. 

Next, we bound the sub-Gaussian norms of $\Delta_k^e$ and $\Delta_k^\phi$. As was done in the proof of Lemma  \ref{lemma:noise_characterization}, we will omit some details in the interest of brevity since the following arguments closely follow arguments given above. We see that the map $w_k \to \delta_k^e$ is Lipschitz continuous with constant $L>0$ where $L = C\Delta t$ for some constant $C>0$ which is independent of $\sigma$. Thus, by \cite[Theorem 2.6.3]{vershynin2018high} there exists a constant $K_1>0$ such that $\norm{\Delta_k^e} \leq K_1 C \Delta t \sigma $. 

Next, we bound the sub-Gaussian norm for $\Delta_k^\phi$. First, we bound the term $\int_{t_k}^{t_k} \hat{J}_kdt = \Delta t \hat{J}_k = \Delta t \cdot R_k \frac{1}{\sigma^2} w_k^T \frac{\partial }{\partial_{\theta_k}} \hat{u}_{\theta_k}$, where we have used the same notation as the proof of Lemma \ref{lemma:grad_estimate}. Using the same arguments as was used for the proof of Lemma \ref{lemma:grad_estimate}, we see that the map $ w_k \to \Delta t \cdot R_k \frac{1}{\sigma^2} w_k^T \frac{\partial }{\partial_{\theta_k}} \hat{u}_{\theta_k}$ is Lipschitz continuous with a Lipschitz constant on the order of $O(\frac{\Delta t}{\sigma^2})$. Thus, using \cite[Theorem 2.6.3]{vershynin2018high} we see that $\norm{\Delta t \cdot R_k \frac{1}{\sigma^2} w_k^T \frac{\partial }{\partial_{\theta_k}} \hat{u}_{\theta_k}}_{\psi_2} = O(\frac{\Delta t}{\sigma})$. Next, we need to bound the sub-Gaussian norm of  $\int_{t_k}^{t_{k+1}}W(t)^TW(t)\bar{\phi}(t)dt$. Now, $W(t) = W(x(t),y_d^\gamma(t), e(t))$ depends on $x(t)$ and $e(t)$ which both depend on $w_k$. However, by Theorem 5.6.2 from \cite{polak2012optimization} our standing Assumptions ensure that the maps $w_k \to x(t)$ and $w_k \to e(t)$ are Lipschitz continuous for each $t \in [t_k, t_{k+1})$. By the continuity of $W$ and Assumption \ref{ass:bounded_trajectories} this allows us to conclude that for each $t \in [t_k, t_{k+1})$ the map $w_k \to W(t)$ is Lipschitz continuous, and thus the map $w_k \to W(t)^TW(t) \bar{\phi}(t)$ is also Lipschitz continuous, since $\bar{\phi}(t)$ is assumed to be bounded by Assumption \ref{ass:bounded_trajectories}. Letting $L$ denote a single common Lipschitz constant for the family of maps $\set{w_k \to W(t)^tW(t) \phi(t)}_{t \in [t_k ,t_{k+1})}$, we then see that the map $w_k \to \int_{t_k}^{t_{k+1}} W(t)^TW(t) \bar{\phi}(t) dt$ is Lipschitz continuous with constant $\Delta t *L$ by integrating the point-wise bound over the length of the interval. Thus, using \cite[Theorem 2.6.3]{vershynin2018high} we see that $\norm{\int_{t_k}^{t_{k+1}} W(t)^TW(t) \bar{\phi}(t) dt}_{\psi_2} = O(\Delta t \sigma)$. Using these bounds and the Triangle inequality we have that $\norm{\delta_k^\phi}_{\psi_2} \leq \norm{\Delta t  \hat{J}}_{\psi_2} + \norm{\int_{t_k}^{t_{k+1}} W(t)^TW(t) \bar{\phi}(t) dt}_{\psi_2} = O( \frac{\Delta t}{\sigma}) +O(\Delta t \sigma)$. For the statement of the Lemma, we drop the $O(\Delta t \sigma)$ since we have assumes that $\sigma^2$ is finite and are interested in small values of $\sigma$.

\subsection{Proof of Theorem \ref{thm:convergence}}
The proof will use the following well-known bound on geometric series:
\begin{equation}\label{eq:geom1}
    \sum_{i =0}^{k-1}\rho^{k-i} \leq 1 + \frac{\rho}{1- \rho}
\end{equation}
However, since $\rho \to 1$ as $\Delta t \to 0$ this bound becomes very large for sampling intervals. To make the dependence on $\Delta t$ more explicit we note that since $e^{-\zeta \Delta t} < 1-\zeta \Delta t$ we have that 
\begin{equation}\label{eq:geom2}
    \frac{\zeta}{1-\zeta} < \frac{1}{\zeta \Delta t}
\end{equation}
Combining the above bounds with \eqref{eq:exp_decay} we have 
\begin{equation}
|\mathbb{E}[X_k]| \leq M \rho^k|X_0| +  (1 +\frac{1}{\zeta \Delta t})K + |\mathbb{E}[\sum_{i=0}^{k-1}\rho^{k-i}\mathcal{M}_k ] |,
\end{equation}
where $K = \sup{ 0\leq j \leq K} |\epsilon_j|$. By Lemma \eqref{eq:total_bias} we have the $K = O(\Delta t^2(1+\sigma +\sigma^2))$ and also that $ |\mathbb{E}[\sum_{i=0}^{k-1}\rho^{k-i}\mathcal{M}_k ] | = 0$, which when combined with the above equation implies \eqref{eq:total_bias}. Next, we characterize the deviation from the mean caused by the $\mathcal{M}_i$ using the inequality from \cite[Theorem 2.6.3]{vershynin2018high} which gives us
\begin{equation}\label{eq:exp_bound}
\mathbb{P}\left\{ \left|\sum_{i=0}^{k-1} \Phi(t_k, t_{i+1})\mathcal{M}_i\right|\geq t^2\right\} \leq 2 \exp\left(\frac{-ct^2}{M_1 \cdot M_2}\right)
\end{equation}
where $c>0$ is a numerical constant, $M_1 = \sup_{0 \leq i \leq k-1}\| \mathcal{M}_i\|_{\psi_2}^2$ and $K_2 = \sum_{i = 0}^{K_2}|\Phi(t_k, t_{i+1})|^2$. Follow steps similar to those used from \eqref{eq:geom1} to \eqref{eq:geom2} one can show that $M_2 \leq 1 + \frac{1}{2\zeta \Delta t}$. Moreover, by Lemma \ref{lemma:noise_characterization} we have $M_1 = O(\Delta t^2(\sigma^2 + \frac{1}{\sigma^2}))= O(\Delta t^2 \sigma^2)$ since we have chosen $\sigma \leq1$. Thus, $M_1 \cdot M_2 = O(\frac{\Delta t}{\alpha \sigma^2})$. Combining this with \eqref{eq:exp_bound} for some constant $C>0$ we have 
\begin{equation}\label{eq:almost_there}
    \mathbb{P}\left\{ \left|\sum_{i=0}^{k-1} \Phi(t_k, t_{i+1})\mathcal{M}_i\right|\geq t^2\right\} \leq 2 \exp\left(\frac{-Ct^2 \zeta \sigma^2}{\Delta t}\right). 
\end{equation}
Next, we equate
\begin{equation}
    \lambda = 2 \exp\left( \frac{-Ct^2 \zeta \sigma^2}{\Delta t}\right)
\end{equation}
which then suggests setting
\begin{equation}\label{eq:final_bullshit}
    t = \sqrt{C}\sqrt{\frac{ \Delta t \ln(\frac{2}{\lambda})}{ \zeta \sigma^2}}.
\end{equation}
Plugging \eqref{eq:final_bullshit} into \eqref{eq:almost_there} and combining the result with \eqref{eq:exp_decay} provides the desired high-confidence bound. 
\subsection{Form of Error Equations}
First, we note that the dynamics of the outputs and their derivatives is given by
\begin{equation}
\dot{\xi} = A\xi + B[b_p(x)  + A_p(x)]u 
\end{equation}
where $b_p(x)$ is the drift term for the plant and $A_p(x)$ is the decoupling matrix for the plant. Thus, at each instant of time the time when the control $u = \hat{u}(\theta, x, \gamma_d^\gamma + Ke)$ is applied derivative of the error equation is given by
\begin{equation}
\dot{e} = \xi - \xi_d= A\xi + B\left[b_p(x)  + A_p(x)\left[ \left( \beta_m(x) +  \sum_{k=1}^{K_1}\theta_1^k \beta_k(x)  \right)+\left( \alpha_m(x) +  \sum_{k=1}^{K_1}\theta_2^k \alpha_k(x) \right) \left( y_d^\gamma + Ke \right) \right] \right]- By_d^\gamma 
\end{equation}
\newpage
\begin{equation}
[
\end{equation}